\documentclass[runningheads,a4paper]{llncs}

\usepackage{wrapfig}
\usepackage{amssymb}
\usepackage{graphicx}
\usepackage{hyperref}

\usepackage{url}
\urldef{\mailsa}\path|{alfred.hofmann, ursula.barth, ingrid.haas, frank.holzwarth,|
\urldef{\mailsb}\path|anna.kramer, leonie.kunz, christine.reiss, nicole.sator,|
\urldef{\mailsc}\path|erika.siebert-cole, peter.strasser, lncs}@springer.com|    

\usepackage{afterpage}
\usepackage{amssymb,amsmath,amstext,float}
\usepackage{graphicx}
\usepackage{color}
\usepackage{multicol}
\usepackage{verbatim}
\usepackage{subfig}
\usepackage{verbatim}
\usepackage{authblk}
\usepackage{nameref}

\usepackage{algorithm}
\usepackage{algorithmic}

\usepackage{array} 
\newcolumntype{P}[1]{>{\centering\arraybackslash}p{#1}} 
\newtheorem{step}{Claim}[section]

\newcommand{\mv}[1]{\mathbf{#1}}

\newcommand{\bstep}{\begin{step}}
\newcommand{\estep}{\end{step}}
\newcommand{\blem}{\begin{lemma}}
\newcommand{\elem}{\end{lemma}}
\newcommand{\brem}{\begin{remark}}
\newcommand{\erem}{\end{remark}}
\newcommand{\bthm}{\begin{theorem}}
\newcommand{\ethm}{\end{theorem}}
\newcommand{\beqn}{\begin{equation}}
\newcommand{\eeqn}{\end{equation}}
\newcommand{\eeq}{\end{equation}}
\newcommand{\beq}{\begin{equation}}
\newcommand{\eitem}{\end{itemize}}
\newcommand{\bitem}{\begin{itemize}}
\newcommand{\eenum}{\end{enumerate}}
\newcommand{\benum}{\begin{enumerate}}

\begin{document}

\mainmatter  

\title{Sampling-based Gradient Regularization for Capturing  Long-Term Dependencies in Recurrent Neural Networks}

\titlerunning{Sampling-based Gradient Regularization in RNNs}

%
%
\author{Artem Chernodub
\thanks{a.chernodub@gmail.com}
\and Dimitri Nowicki
\thanks{nowicki@nnteam.org.ua}}
\authorrunning{A.N. Chernodub, D.V. Nowicki}

\institute{Institute of MMS of NASU, Center for Cybernetics, 42 Glushkova ave.,  Kiev, Ukraine 03187}

%
%

\toctitle{Lecture Notes in Computer Science}
\tocauthor{Authors' Instructions}
\maketitle

\bibliographystyle{unsrt} 

\begin{abstract}
Vanishing (and exploding) gradients effect is a common problem for recurrent neural networks which use backpropagation method for calculation of derivatives. We construct an analytical framework to estimate a contribution of each training example to the norm of the long-term components of the target function’s gradient and use it to hold the norm of the gradients in the suitable range. Using this subroutine we can construct mini-batches for the stochastic gradient descent (SGD) training that leads to high performance and accuracy of the trained network even for very complex tasks. To check our framework experimentally we use some special synthetic benchmarks for testing RNNs on ability to capture long-term dependencies. Our network can detect “links” between events in the (temporal) sequence at the range ~100 and longer.

\end{abstract}

\section{Introduction}
Recurrent Neural Networks (RNNs) are known as universal approximators of dynamic systems \cite{Hava-1997}. Since RNNs are able to simulate any open dynamical system, they have a broad spectrum of applications such as time series forecasting \cite{Cardot-2011}, control of plants \cite{Prokhorov-2008}, language modeling \cite{Mikolov-2010}, speech recognition, neural machine translation \cite{Cho-2014} and other domains. 
The easiest way to create an RNN is adding the feedback connections to the hidden layer of multilayer perceptron. This architecture is known as Simple Recurrent Network (SRN). Despite of the simplicity, it has rich dynamical approximation capabilities mentioned above. However, in practice training of SRNs using first-order optimization methods is difficult \cite{Bengio-1994}. The main problem is well-known “vanishing/exploding gradients effect” that prevents capturing of long-term dependencies in data. Vanishing gradients effect is a common problem for recurrent and deep neural networks with sigmoid-like activation functions which uses a backpropagation method for calculation of derivatives. Hochreiter and Schmidhuber designed a set of special synthetic benchmarks for testing RNNs on ability to capture long-term dependencies \cite{Hochreiter-1997}. They showed that ordinary SRNs are very ineffective to learn correlations in sequential data if distance between the target events is more than 10 time steps. 
The solution could be using more advanced second-order optimization algorithms such as Extended Kalman Filter, LBFGS, Hessian-Free optimization \cite{Martens-2011}, but they require much more memory and computational resources for state-of-the-art networks. We also mention such an alternative to temporal neural networks as hierarchical
sequence processing with auto-associative memories \cite{kussul1991multilevel}. The mainstream solution for the gradient control problem  is based on more complex architectures such as LSTM \cite{Hochreiter-1997} or GRU \cite{Cho-2014} networks. However, training the SRN's for catching long-term dependencies is  highly desirable at least for better understanding of underlying processes of the training inside the recurrent and deep neural networks. Also, SRNs are more compact and fast working models of RNNs in comparison with  LSTMs that is very important for implementation to mobile and embedded devices. Recent research shows the ability to train SRNs for long term dependencies up to 100 time steps and more using several new techniques \cite{Martens-2011}, \cite{Pascanu-2012}.  In this paper we propose a new method to perform the gradient regularization by selection of proper samples in dataset.

\section{Backpropagation Mechanism Revisited} 

Consider a SRN that at each time step $k$ receives an external input $\mv{u}(k)$, previous internal state $\mv{z}(k-1)$ and produces output $\mv{y}(k+1)$:

\begin{equation}
  \begin{array}{l}
\mv{a}(k)=\mv{u}(k)\mv{w}_{in} 
+\mv{z}(k-1)\mv{w}_{rec} +\mv{b},\\ \mv{z}(k)=f(\mv{a}(k)),\\
\mv{y}(k+1)=g(\mv{z}(k)\mv{w}_{out} ),
  \end{array}
  \label{eq_srn}
  \end{equation}

where $\mv{w}_{in} $ is a matrix of input weights, $\mv{w}_{rec} $ is matrix of recurrent weights, $\mv{w}_{out} $ is 
matrix of output weights, $\mv{a}(k)$ is known as ``presynaptic activations'', $\mv{z}(k)$ is 
a network's state, $f(\cdot )$ and $g(\cdot )$ are nonlinear activation functions 
for hidden and output layer respectively. In this work we always use ${tanh}$ function 
for hidden layer and optionally ${softmax}$ or $linear$ function depending 
on the target problem (classification or regression) for output layer. 

The dynamic error derivative is a sum of immediate derivatives: $\frac{\partial E}{\partial \mv{w}} =\sum _{n=1}^{h}\frac{\partial E}{\partial \mv{w}(k-n)}$, 
where $n=1,...,h$, where $h$ is BPTT's truncation depth. An intermediate variable  $\mv{\delta} 
\equiv \frac{\partial E}{\partial \mv{a}} $  called a ``local gradients'' or simply ``deltas'' 
is usually introduced for convenience,

\begin{equation}
\delta (k-h)=\delta (k-h+1)\mv{w}_{rec}^{T} diag(f'(\mv{a}(k-h))).
\label{eqd4}
\end{equation}

Equation \eqref{eqd4} may be rewritten using Jacobian matrix $\mv{J}(n)=\frac{\partial \mv{z}(n)}{\partial \mv{z}(n-1)} $:

\begin{equation}
\mv{\delta} (k-h)=\mv{\delta} (k-h+1)\mv{J}(k-h).
\label{eqd5}
\end{equation}

Now we can use an intuitive understanding of exploding/vanishing gradients problem  that was deeply investigated in classic \cite{Bengio-1994} and modern papers \cite{Pascanu-2012}. As it can be seen from \eqref{eqd5},  norm of the backpropagated deltas is strongly dependent on norm of the Jacobians. Moreover, they actually are product of Jacobians: $\mv{\delta} (k-h)=\mv{\delta} (n)\mv{J}(k)\mv{J}(k-1)...\mv{J}(k-h+1)$.
The ``older'' deltas are, the more Jacobian matrices were multiplied. If norm of Jacobians 
are more than 1 if the gradients will grow exponentially in most cases. It refers to the RNN's behavior 
where long-term components are more important than short-term ones. Vice versa, if 
norm of Jacobians are less than 1, this leads to vanishing gradients and ``forgetting'' 
the long-term events. In \cite{Pascanu-2012} a universal ``gradient regularization'' approach that forces the gradient norm to stay in a stable range via modification of the training objective function proposed. However, they used a complex regularizer to preserve norm in the relevant direction.

\section{Differentiation of the gradient's norm}

Let $\mv{d}=\{ \mv{u}_{1} ;\mv{t}_{1} ;...;\mv{u}_{N} ,\mv{t}_{N} 
\} $   be a minibatch with   $N_D$ training examples. We do forward and back propagation in the network for this minibatch and we get the difference (correction) vector  $d\mv{w}$. Let\'s check how $d\mv{w}$ influences on gradient vanishing or explosion. Let $ \mv{w}_{rec}^{l} $ be a weight matrix for recurrent layer of the SRN at the current iteration $l$ of weight update. Suppose we have made the back and forward pass, so   $ \mv{w}_{rec}^{l} $ is a correction for the recurrent layer such that $ \mv{w}_{rec}^{l+1} =   \mv{w}_{rec}^{l} + d\mv{w}_{rec}$.

Consider a function $S(\mv{w}_{rec}^{(l)})$ that is equal to squared Euclidian norm of \eqref{eqd4} for iteration $l$:  
\begin{equation}
S(\mv{w}_{rec}^{(l)} )=\frac{1}{2}\left\| \mv{\delta} (k-h,\mv{w}_{rec} )\right\| ^{2}_{2}.
\label{eq320}
\end{equation}

Since $\left\| d\mv{w}_{rec}^{(l)}\right\|^{2}_{2}$ is supposed to be small, we can use Taylor expansion of \eqref{eq320} at the current point of weight matrix space:
	
\begin{equation}
S(\mv{w}_{rec}^{(l+1)} )=S(\mv{w}_{rec}^{(l)} + dS + o(\left\| d\mv{w}_{rec}^{(l)}\right\|^{2}_{2})).
\label{eq321}
\end{equation}

\begin{lemma}Linear term $dS$  in (\ref{eq321}) could be expressed as a scalar product of the auxilary vectors $\mv{g}$ and $d\mv{g}$,

\begin{equation}
dS=(\mv{g},d\mv{g}),
\label{eq322}
\end{equation} 
where
\begin{equation}
  \begin{array}{l}
\\
\mv{g}=\left(\prod _{i=h}^{1}diag(f'(\mv{a}(k-i+1)))\mv{w}_{rec}  \right)\delta (k), \\

d\mv{g}=\sum _{i=1}^{h}\left(\left(\prod _{j=h}^{1}diag\left[f'(\mv{a}(k-j+1))\right]\mv{v} \right)
\mv{\delta} (k)\right),\\
\mv{v}=d\mv{w}_{rec} ,_{} \mathrm{if}  \ i=j;_{} \mv{v}=\mv{w}_{rec} ,_{} if \  i\ne j.
  \end{array}
  \label{eq323}
  \end{equation}
\end{lemma}

\begin{proof}

Using \eqref{eqd4}, \eqref{eqd5} we get  $\delta ^{(l)} (k-h)$:

\begin{equation}
\mv{\delta} (k-h)=\delta (k)\mv{w}_{rec}^{T} diag(f'(\mv{a}(k-1)))...\mv{w}_{rec}^{T} diag(f'(\mv{a}(k-h+1))).
\label{eq325}
\end{equation}

We introduce a  $\mv{D}_{n}$ notation as follows:

\begin{equation}
\mv{D}_{n} \equiv diag(f'(\mv{a}(n))). 
\label{eq326}
\end{equation}

Now \eqref{eq325} becomes:

\begin{equation}
\delta ^{(l)} (k-h)=\delta ^{(l)} (k)\mv{w}_{rec}^{T (l)} \mv{D}_{k-1} w_{rec}^{T (l)} \mv{D}_{k-2} ...\mv{w}_{rec}^{T (l)} \mv{D}_{k-h} .
\label{eq327}
\end{equation}

Auxillary vector $\mv{g}$ is a transposed \eqref{eq327}:

\begin{equation}
\mv{g}=\delta ^{(l)} (k-h)^{T}.
\label{eq328}
\end{equation}

Since $(\mv{AB})^{T} =\mv{B}^{T}\mv{A}^{T} $, \eqref{eq327} and \eqref{eq328} lead to :
\begin{equation}
\mv{g}=\mv{D}_{k-1}\mv{w}_{rec}^{(l)}\mv{D}_{2}\mv{w}_{rec}^{(l)} ...\mv{D}_{k-h}\mv{w}_{rec}^{(l)}\delta ^{(l)} (k).
\label{eq329}
\end{equation} 

Since $\left\| x\right\| _{2} =\left\| x^{T} \right\| _{2} $, regarding \eqref{eq328} and \eqref{eq329}, the function $S(\mv{w}_{rec}^{(l)})$ \eqref{eq320} has an equivalent form:

\begin{equation}
S(\mv{w}_{rec}^{(l)} )=\frac{1}{2} \left\| \mv{g}\right\| _{2}^{2}.
\label{eq330}
\end{equation} 

Now we get the differential $dS$ of \eqref{eq330}:

\begin{equation}
dS=(\mv{g},d\mv{g}),
\label{eq331}
\end{equation}where the vector $\mv{g}$ in \eqref{eq331} is obtained from \eqref{eq329} that is equivalent \eqref{eq323}. Also, we get vector $d\mv{g}$ in \eqref{eq331} by differentiating vector $\mv{g}$ in \eqref{eq329} as follows:

\begin{equation}
d\mv{g}=\sum _{i=1}^{h}\delta(k)\mv{D}_{k-h}\mv{w}_{rec}^{(l)} ...\mv{D}_{i}d\mv{w}_{rec}^{(l)}...\mv{D}_{k-1}\mv{w}_{rec}^{(l)},
\end{equation}

that is the same as \eqref{eq323} up to usage notation $\mv{D}_{n}$. The lemma is proven.\end{proof}

We have to figure out the direction of change of the (Euclidean) norm of $\delta (k-h)$ since the gradient is propagated for $h$ steps back at time step $k$ where the correction $d\mv{w}_{rec} $ is used.

\begin{theorem}
The condition $dS>0$ is sufficient to increase the norm of $\left\| \delta^{(l+1)} (k-h)\right\| _{2}$ comparing to $\left\| \delta ^{(l)} (k-h)\right\| _{2} $ at the next iteration $l+1$ of weight correction then the correction matrix $d\mv{w}_{rec}$ is used. $dS$ here is defined by \eqref{eq322} and $d\mv{w}_{rec} $ is contained in $dS$, and $\mv{w}_{rec}^{(l+1)} = \mv{w}_{rec}^{(l)}+d\mv{w}_{rec} $. Similarily, $dS < 0$ is a sufficient condition for decrease of $\left\| \delta (k-h, \mv{w}_{rec}^{(l+1)} )\right\| _{2}$.
\end{theorem}

\begin{proof}

Let's compare values $S(\mv{w}_{rec}^{(l)} )$ and $S(\mv{w}_{rec}^{(l+1)} )$ that correspond to current iteration $l$ and next one $l+1$ of the weight update. We use Taylor expansion \eqref{eq321} and the Lemma 1. From \eqref{eq321} follows that the sign of \eqref{eq322} defines a direction of change for the Euclidean norm of $\delta(k-h)$ between the iterations $l$ and $l+1$ and absolute value $|dS|$ defines the magnitude of this change.\end{proof} 

Idea of our sampling-based gradient regularization algorithm is selection of ``proper'' samples of data for training. Using Theorem 1 we can clearly find out an impact of each mini-batch on norm of backpropagated gradients. 

We introduce auxiliary variable called $Q$-factor that measures how much the norm of the gradient is decreased or increased during the backpropagation. For ideal catching of long-term dependencies Q-factor must be close to 0.
\begin{equation}
Q(\mv{\delta}, h)=\log_{10} (\frac{\left\| \mv{\delta} (k)\right\| }{\left\| \mv{\delta} (k-h)\right\| } ).
\label{eq27}
\end{equation}

Here we use the simplest and the most straightforward method: we watch a norm of the gradients; if the norm becomes too  small, we omit mini-batches of data such that decrease this norm. Vice versa, if norm becomes very large, we skip mini-batches increasing this norm even more. Also, note that it is better to skip minibatches with large $|dS|$: they  can cause high "leaps" of the gradient norm and therefore its self-oscillations. 

\begin{algorithm}[tb]
   \caption{Algorithm of sampling-based gradient regularization}
   \label{algmain}
\begin{algorithmic}
   \STATE {\bfseries Input:}  training data $\{ \mv{U},\mv{T}\}$ , $r_0>0$.
   
   \FOR{each minibatch $\mv{u}_i; \mv{d}_i$ with $N_{D}$ vectors}

   \STATE   calculate $dS$ (\ref{eq322}), if $|dS|>0$
		\STATE \textbf{continue}
   \STATE make forward and backward propagation
		\STATE calculate $Q(\mv{\delta}, h)$ (\ref{eq27})
		 \IF{ $Q(\mv{\delta}, h)\in [Q_{min}; Q_{max}]$}
			\STATE use current minibatch for training 
			\ELSE 
			 \IF { $(Q(\mv{\delta},h)<Q_{min}\  \mathbf{and} \ dS>0)\  \mathbf{or}\ (Q(\mv{\delta},h)>Q_{max}\ \mathbf{and} \ dS<0)$}
			\STATE use current minibatch for training 
			\ELSE
			\STATE \textbf{continue}
			\ENDIF 
			\ENDIF

   \ENDFOR
\end{algorithmic}
\end{algorithm}

\section{Experiments}

We follow \cite{Pascanu-2012} and use the following synthetic problems for catching long-term dependencies: ``Adding'', ``Multiplication'', ``Temporal order'', ``Temporal order 3-bit''. Two sets containing 10 SRN with 100 hidden units each were initialized by random values  and saved. Thus, for different training methods initial weights of neural networks were 
the same. ``Safe'' range $\left[Q_{min} ;Q_{max} \right]$ for \eqref{eq27} was set to $\left[-1;1
\right]$.  

We use SGD optimization, training speed $\alpha
=10^{-5} ...10^{-3}$, momentum $\mu =0.9$, size of mini-batch is 10. Train / validation / test datasets contains 20,000 / 1000 / 10,000 samples respectively. After each epoch, network's performance is tested on validation dataset; network that has the best performance on the validation dataset is tested  on the test dataset, this result is recognized as the final result. We trained SRNs during 2000 epochs, each epoch consists 50 iterations, i.e. 100,000 corrections of weights at all.

 Weights  were initialized by small values from Gaussian distribution with zero mean and standard 
deviation $\sigma$. On Fig. \ref{fig35} average norms of gradients as function of  backpropagation depth (before training, further referred as initial gradients) are graphed for different values $\sigma $ for the  ``Temporal order problem''. We see that good initialization of weights is very important because  vanishing/exploding gradients has monotonous flow in most cases because gradients  are propagated through the same matrix of recurrent weights..

\begin{figure*}[htbp]
\begin{center}
\subfloat[$\sigma$= 0.01]{\includegraphics[width=4cm]{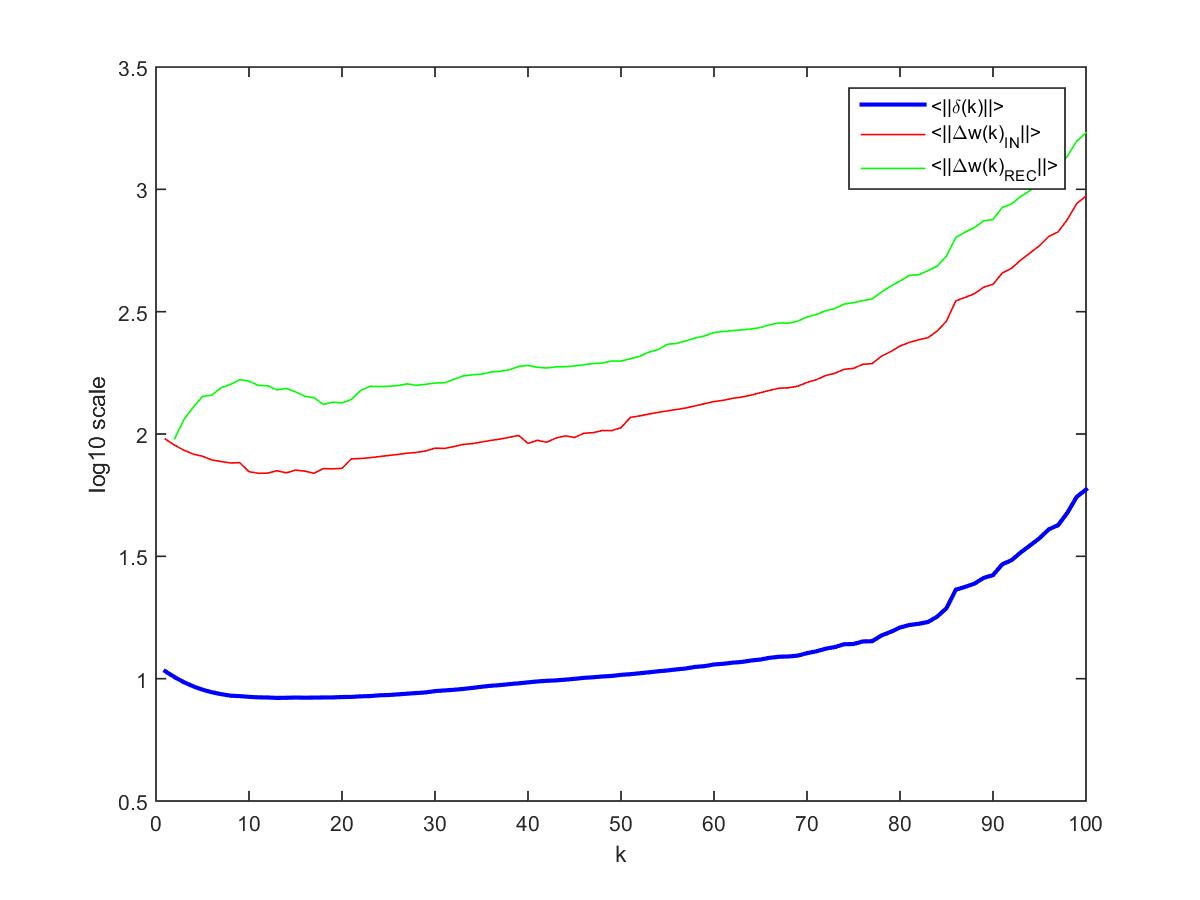}}
\subfloat[$\sigma$= 0.005]{\includegraphics[width=4cm]{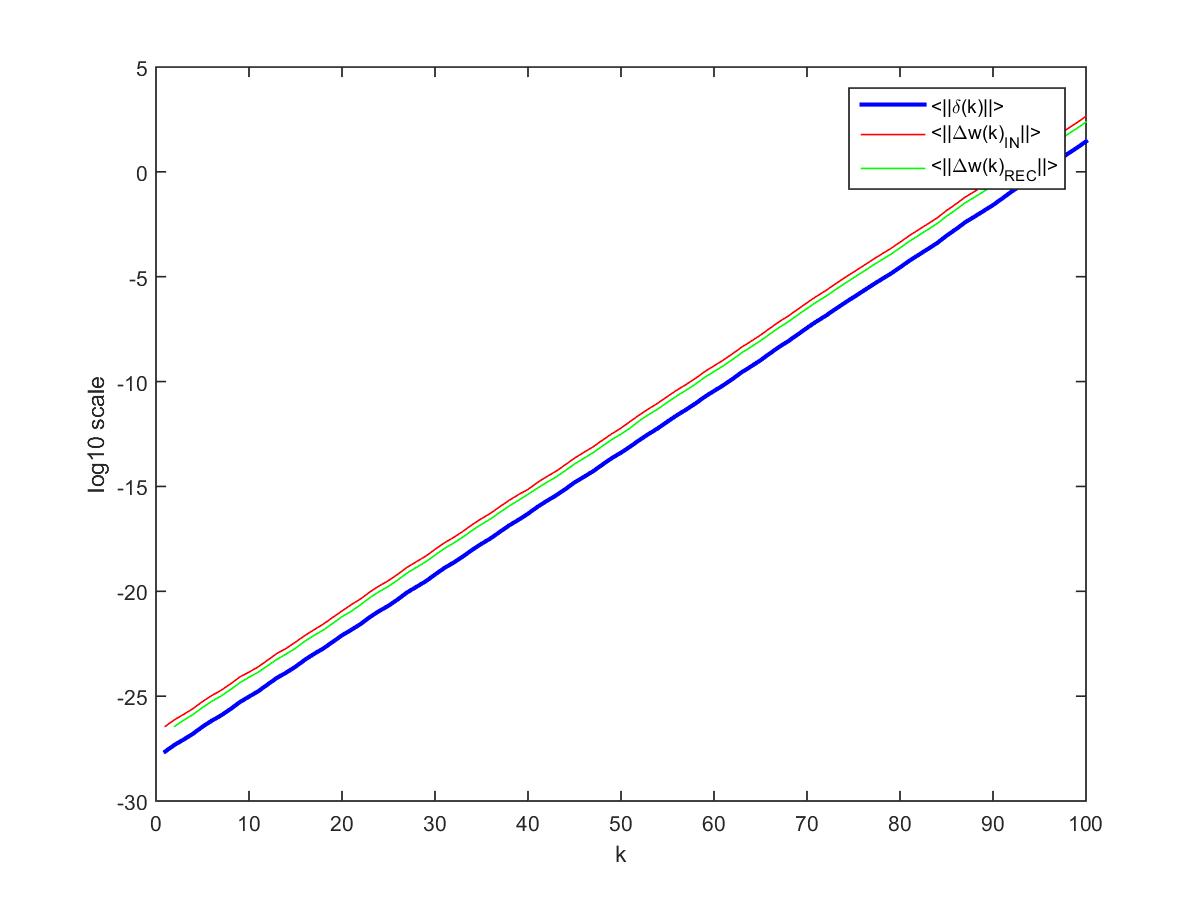}}
\subfloat[$\sigma$= 0.02]{\includegraphics[width=4cm]{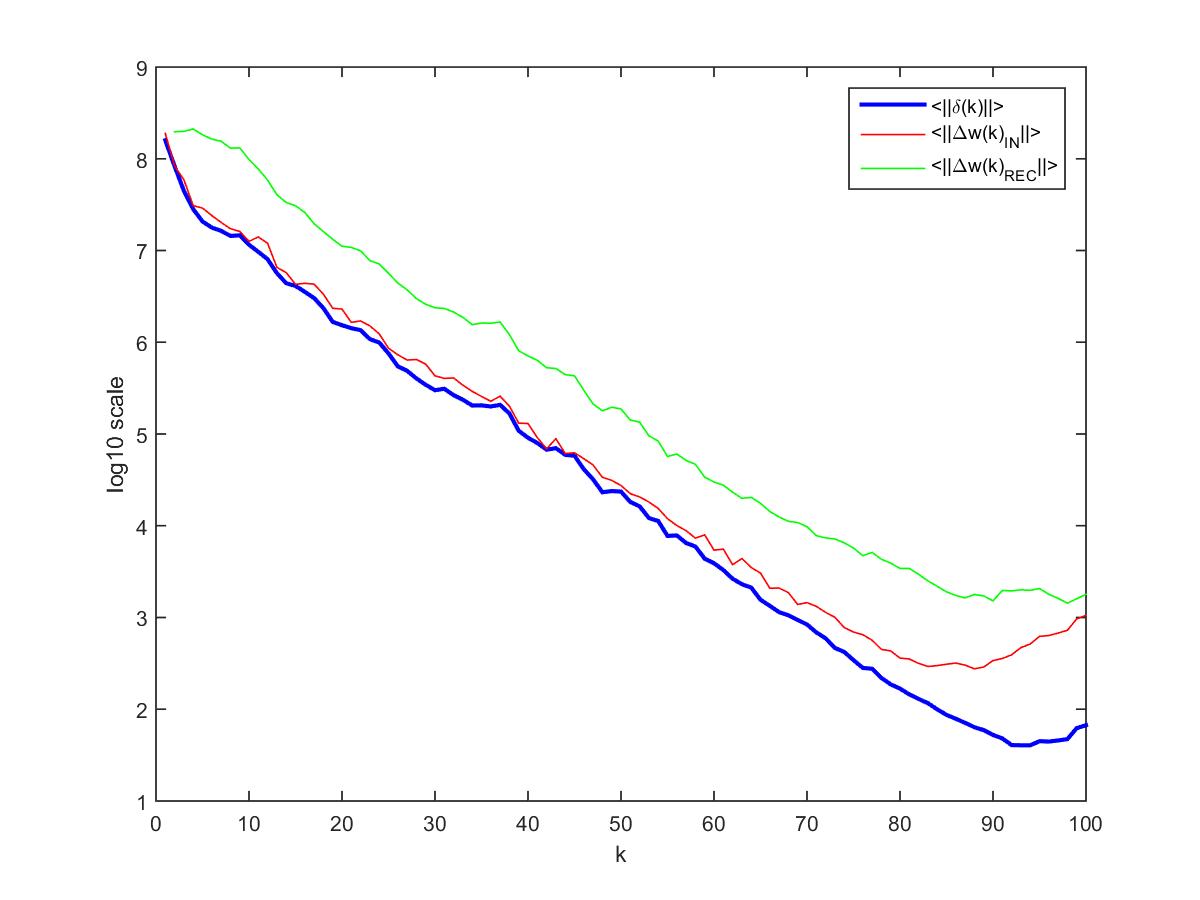}}

\end{center}

\caption{\label{fig35} Average norms of backpropagated (initial) gradients for 
SRNs, horizon BPTT $h=100$.}
\end{figure*}

Each chart at Fig \ref{fig35} contains three curves: average norms of local gradients $\delta 
(k)$ (blue) and average norms of gradients $\Delta \mv{w}(k)_{in} \equiv \frac{
\partial E}{\partial \mv{w}_{in} } $ and $\Delta \mv{w}(k)_{rec} \equiv \frac{\partial E}{
\partial \mv{w}_{rec} } $ (red and green). From the graphs at Fig. 3.5 one can 
ensure on practice that to control the norms $\frac{\partial E}{\partial \mv{w}}$   which 
actually make changes to the weights and are under the main scope of our interest 
it is enough to control the norms of local gradients $\mv{\delta} (k)$ because they are 
highly correlated.

Finally we used $\sigma =0.01$ as in \cite{Pascanu-2012}. However, proper initialization doesn't guarantee successful training. Particular  case of forward and backward dynamics (norms of the backpropagated gradients are depicted on the top, mean and  median activation values) during training of SRN network is shown on 
Fig. \ref{fig3_6}. SRN that is depicted on Fig.\ref{fig3_6} was initialized with $\sigma =0.01$ and initial norms of backpropagated gradients were similar to Fig. \ref{fig35} a). However, after 500 iterations 
we got norm of gradients less than $10^{-7} $ for $h = 100$.  After that almost all the time neural networks had small gradients in the range $10^{-7} ...10^{-8} $ . From the graphs on Fig. \ref{fig3_6}, on the left, we see that area of small  gradients is related to area of saturation for neuron's activations. This is a symptom 
of bad network abilities for successful training and obtaining good generalization 
properties. 

\begin{figure}[ht]
\begin{center}
\subfloat[Grad. reg. OFF]{\includegraphics[height=6cm]{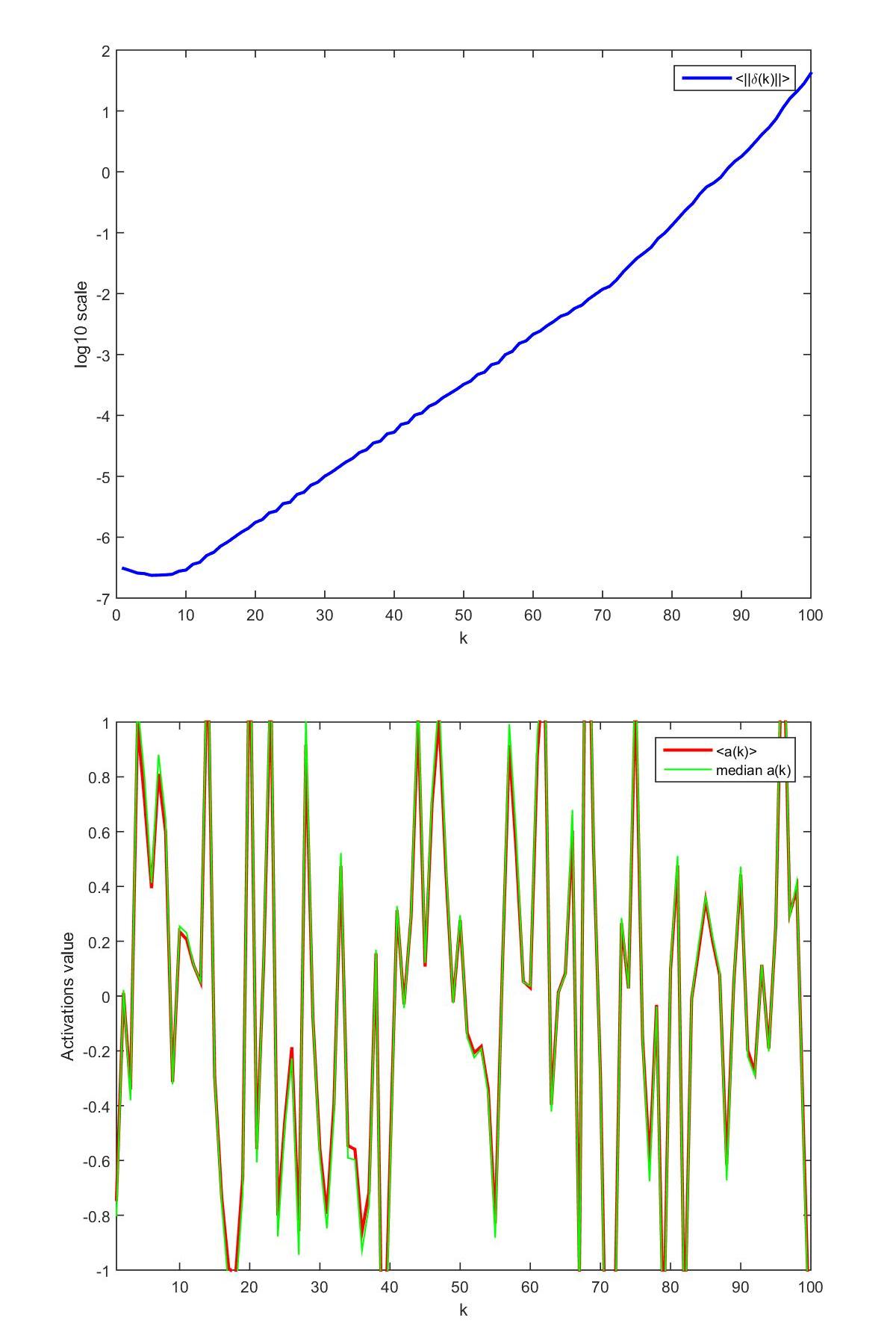}}
\subfloat[Grad. reg. ON]{\includegraphics[height=6cm]{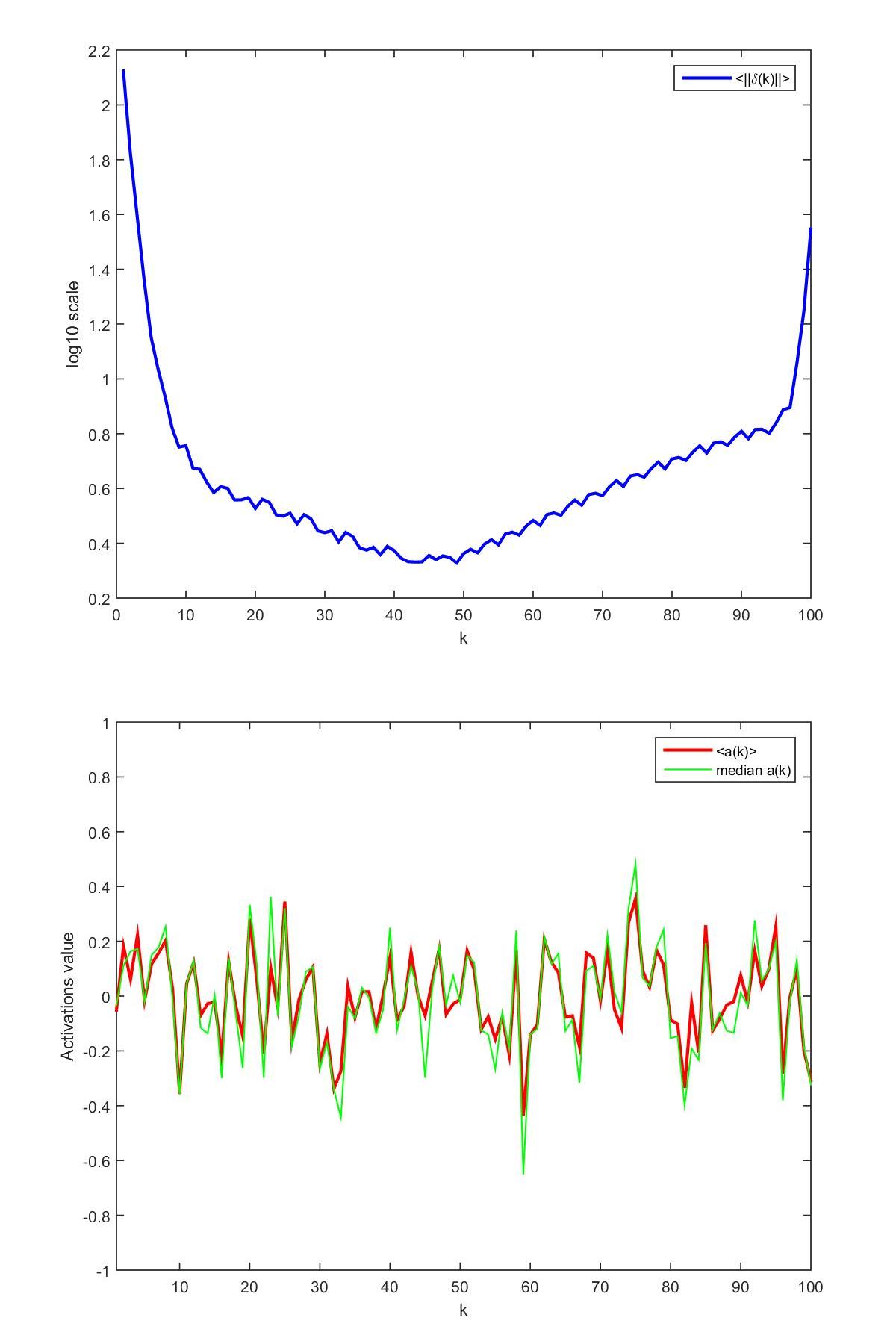}}

\caption{ \label{fig3_6}Evolution of internal dynamics inside the SRN 
during training without gradients regularization (a) and with sampling-based gradients regularization (b). On  both (a) and (b) upper graphs are mean norms 
in time of backpropagated via  BPTT local gradients $\mv{\delta} (k)$; lower ones are mean 
and median values of activations $\mv{a}(k)$. }
\end{center}
\end{figure} 


\begin{table*}[h]
  \centering
\caption{\label{tab:tab1} Accuracies of trained SRNs for synthetic problems which have long-term dependencies without gradient regularization (traditional training method)  and with sampling-based gradient regularization (proposed method), for T = 100 and 150.}
\begin{small}
\begin{center}
\begin{tabular}{ |P{2.5cm}| P{1cm} | P{1cm} | P{1cm}| P{1cm} | P{1cm} | P{1.0cm}| P{1.0cm} | P{1.0cm} |P{1.0cm}}
\hline
& \multicolumn{2}{c|}{Adding} & \multicolumn{2}{c|}{Multiplication} & \multicolumn{2}{c|}{Temporal order} & \multicolumn{2}{c|}{Temporal order 3-bit} \\
\cline{2-3}
\cline{4-5}
\cline{6-7}
\cline{8-9}
& best & mean & best & mean & best & mean & best & mean \\
\hline
T=100, grad. reg. OFF& 99\% & 68\% & \textgreater99\% & 72\% & 96\% & 44\% & 99\% & 50\% \\
\hline
T=100, grad. reg. ON  & \textgreater99\% & 96\% & \textgreater99\% & 68\% & \textgreater99\% & 60\% & \textgreater99\% & 62\% \\

\hline
T=150, grad. reg. OFF& 34\% & 11\% & N/A & N/A & 51\% & 30\% & 32\% & 24\% \\
\hline
T=150, grad. reg. ON & 47\% & 13\% & N/A & N/A & 72\% & 42\% & 37\% & 30\% \\
\hline

\end{tabular}
\end{center}
\end{small}
\end{table*}

Using our sampling-based gradients regularization allows to refine the quality of training (Table \ref{tab:tab1}). For lengths $T=100$ and $T=150$ improvement is 10-20\% in average. Samples rejected by the algorithm during the training not necessarily are lost for using in future training process because they may be used when 
network is in ``safe region'' or we may need to change norms of gradients in the 
opposite direction.
\section*{Acknowledgments}

We thank FlyElephant (http://flyelephant.net) and Dmitry Spodarets for computational resources kindly given for our experiments.

\section{Conclusion}

We provided a novel solution of the problem of exploding and vanishing gradient effects, applied to the Simple Recurrent Networks. We analytically derived sufficient conditions on increase and decrease the Euclidean vector norm of backpropagated gradients for SRNs. Using this theorem we designed the algorithm that controls norm of the gradient operating solely with presence of the minbatches in the training sequence. This framework was tested for long-term prediction on a comprehensive set of appropriate benchmarks. Resulting accuracy outperforms best known SRN learning algorithms by 10-20\%. This paradigm could be generalized to deep and multi-layered recurrent networks, that is a subject of our future research. 
\bibliography{grad-bib}
\end{document}